\documentclass{article} 
\usepackage{iclr2018_conference,times}
\usepackage{hyperref}
\usepackage{url}

\title{Auto-Encoding Sequential Monte Carlo}

\author{Tuan~Anh Le$^\dag$, Maximilian Igl$^\dag$, Tom Rainforth$^\ddag$, Tom Jin$^{\dag,\S}$, Frank Wood$^\dag$ \\
$^\dag$Department of Engineering Science, University of Oxford \\
$^\ddag$Department of Statistics, University of Oxford \\
$^\S$Department of Statistics, University of Warwick \\
\texttt{\{tuananh,igl,jin,fwood\}@robots.ox.ac.uk, ~ rainforth@stats.ox.ac.uk}
}

\iclrfinalcopy 

\usepackage[utf8]{inputenc}
\usepackage[T1]{fontenc}
\usepackage{hyperref}
\usepackage{url}
\usepackage{booktabs}
\usepackage{amsfonts}
\usepackage{amsthm}
\usepackage{multirow}
\usepackage{nicefrac}
\usepackage{microtype}
\usepackage{amsmath}
\usepackage{graphicx}
\usepackage{caption}
\usepackage{subcaption}
\usepackage{mathtools}
\usepackage{cancel}
\usepackage{todonotes}
\usepackage[ruled,vlined]{algorithm2e}
\usepackage[acronym,smallcaps,nowarn,section,nogroupskip,nonumberlist]{glossaries}
\usepackage{wrapfig}
\usepackage[section]{placeins}

\glsdisablehyper{}
\newacronym{VAE}{vae}{variational auto-encoder}
\newacronym{AESMC}{aesmc}{auto-encoding sequential Monte Carlo}
\newacronym{IS}{is}{importance sampling}
\newacronym{IWAE}{iwae}{importance weighted auto-encoder}
\newacronym{SMC}{smc}{sequential Monte Carlo}
\newacronym{SSM}{ssm}{state-space model}
\newacronym{SGA}{sga}{stochastic gradient ascent}
\newacronym{SGD}{sgd}{stochastic gradient descent}
\newacronym{ELBO}{elbo}{evidence lower bound}
\newacronym{KL}{kl}{Kullback-Leibler}
\newacronym{LSTM}{lstm}{long short-term memory}
\newacronym{AD}{ad}{automatic differentiation}
\newacronym{SPSA}{spsa}{simultaneous perturbation stochastic approximation}
\newacronym{CG-SPSA}{cg-spsa}{computational graph SPSA}
\newacronym{MML}{mml}{maximum marginal likelihood}
\newacronym{REINFORCE}{reinforce}{REINFORCE}
\glsunset{REINFORCE}
\newacronym{ADAM}{adam}{ADAM}
\glsunset{ADAM}
\newacronym{GRU}{gru}{gated recurrent unit}
\newacronym{MLP}{mlp}{multilayer perceptron}
\newacronym{MAP}{map}{maximum a-posteriori}
\newacronym{KDE}{kde}{kernel density estimation}
\newacronym{EM}{em}{expectation maximization}
\newacronym{MC}{mc}{Monte Carlo}
\newacronym{ALT}{alt}{alternating \textsc{elbo}s}
\newacronym{SNR}{snr}{signal-to-noise ratio}
\newacronym{VRNN}{vrnn}{Variational Recurrent Neural Network}
\newacronym{LGSSM}{lgssm}{linear Gaussian state space model}
\newacronym[firstplural=recurrent neural networks, plural=RNNs]{RNN}{rnn}{recurrent neural network}

\newcommand{\given}{\lvert}
\DeclareMathOperator{\E}{\mathbb{E}}
\DeclareMathOperator{\Var}{\mathrm{Var}}

\DeclareMathOperator{\ELBO}{\acrshort{ELBO}}
\DeclareMathOperator{\SNR}{\acrshort{SNR}}
\DeclareMathOperator*{\argmax}{arg\,\!max}

\newcommand{\KL}[2]{\acrshort{KL}\left(#1 \middle| \middle| #2\right)}

\newtheorem{claim}{Claim}
\newtheorem{definition}{Definition}

\newtheorem{proposition}{Proposition}

\begin{document}

\maketitle

\begin{abstract}
    We build on \gls{AESMC}:\footnote{This work builds upon an earlier preprint~\citep{le2017auto} along with the independent, simultaneously developed, closely related, work
    of \citet{maddison2017filtering} and \citet{naesseth2017variational}.} a method for model and proposal learning based on maximizing the lower bound to the log marginal likelihood in a broad family of structured probabilistic models.
    Our approach relies on the efficiency of \gls{SMC} for performing inference in structured probabilistic models and the flexibility of deep neural networks to model complex conditional probability distributions.
    We develop additional theoretical insights and experiment with a new training procedure which can improve both model and proposal learning.
    We demonstrate that our approach provides a fast, easy-to-implement and scalable means for simultaneous model learning and proposal adaptation in deep generative models.
\end{abstract}

\section{Introduction}

We build upon \gls{AESMC}~\citep{le2017auto}, a method for model learning that itself builds on \glspl{VAE}~\citep{kingma2014auto,rezende2014stochastic} and \glspl{IWAE}~\citep{burda2016importance}.
\Gls{AESMC} is similarly based on maximizing a lower bound to the log marginal likelihood,
but uses \gls{SMC}~\citep{doucet2009tutorial} as the underlying marginal likelihood
estimator instead of \gls{IS}.  For a very wide array of models, particularly
those with sequential structure, \gls{SMC} forms a substantially more powerful inference
method
than \gls{IS}, typically returning lower variance estimates for
the marginal likelihood.  Consequently, by using \gls{SMC} for its marginal
likelihood estimation, \gls{AESMC} often leads to improvements in model learning compared
with \glspl{VAE} and \glspl{IWAE}.  We provide experiments on 
structured time-series data that show that \gls{AESMC} based learning was able to learn useful representations of the latent space for both reconstruction and prediction more effectively
than the \gls{IWAE} counterpart.

\Gls{AESMC} was introduced in an earlier preprint~\citep{le2017auto} concurrently with the closely related methods of \citet{maddison2017filtering,naesseth2017variational}.  In this work we take these ideas
further by providing new theoretical insights for the resulting \glspl{ELBO}, extending these
to explore the relative efficiency of different approaches to proposal learning, and using
our results to develop a new and improved training procedure. 
In particular, we introduce a method for expressing the gap between an \gls{ELBO} and the log marginal likelihood as a \gls{KL} divergence between two distributions on an extended sampling space.
Doing so allows us to investigate the behavior of this family of algorithms when the objective is maximized perfectly, which occurs only if the \gls{KL} divergence becomes zero.
In the \gls{IWAE} case, this implies that the proposal distributions are equal to the posterior distributions under the learned model.
In the \gls{AESMC} case, it has implications for both the proposal distributions and
the intermediate set of targets that are learned.
We demonstrate that, somewhat counter-intuitively, using lower variance estimates for
the marginal likelihood can actually be harmful to proposal learning.  
Using these insights, we experiment with an adaptation to the
\gls{AESMC} algorithm, which we call \emph{alternating} \glspl{ELBO}, that uses
different lower bounds for updating the model parameters and proposal parameters.
We observe that this adaptation can, in some cases, improve model learning and proposal adaptation.

\section{Background}

\subsection{State-Space Models}

\Glspl{SSM} are probabilistic models over a set of latent variables $x_{1:T}$ and observed variables $y_{1:T}$.
Given parameters $\theta$, a \gls{SSM} is characterized by an initial density $\mu_{\theta}(x_1)$, a series of transition densities $f_{t, \theta}(x_t \given x_{1:t - 1})$, and a series of emission densities $g_{t, \theta}(y_t \given x_{1:t})$ with the joint density being $p_{\theta}(x_{1:T}, y_{1:T}) = \mu_{\theta}(x_1) \prod_{t = 2}^T f_{t,\theta}(x_t \given x_{1:t - 1}) \prod_{t = 1}^T g_{t,\theta}(y_t \given x_{1:t})$.

We are usually interested in approximating the posterior $p_{\theta}(x_{1:T} \given y_{1:T})$ or the expectation of some test function $\varphi$ under this posterior $I(\varphi) := \int \varphi(x_{1:T}) p_{\theta}(x_{1:T} \given y_{1:T}) \,\mathrm dx_{1:T}$.
We refer to these two tasks as inference.
Inference in models which are non-linear, non-discrete, and non-Gaussian is difficult and one must resort to approximate methods, for which \gls{SMC} has been shown to be one of the most powerful approaches \citep{doucet2009tutorial}.

We will consider model learning as a problem of maximizing the marginal likelihood $p_{\theta}(y_{1:T}) = \int p_{\theta}(x_{1:T}, y_{1:T}) \,\mathrm dx_{1:T}$ in the family of models parameterized by $\theta$.

\subsection{Sequential Monte Carlo}

\gls{SMC} performs approximate inference on a sequence of target distributions $(\pi_t(x_{1:t}))_{t = 1}^T$.
In the context of \glspl{SSM}, the target distributions are often taken to be $(p_{\theta}(x_{1:t} \given y_{1:t}))_{t = 1}^T$.
Given a parameter $\phi$ and proposal distributions $q_{1, \phi}(x_1 \given y_1)$ and $(q_{t, \phi}(x_t \given y_{1:t}, x_{1:t - 1}))_{t = 2}^T$ from which we can sample and whose densities we can evaluate, \gls{SMC} is described in Algorithm~\ref{alg:background/smc}.

Using the set of weighted particles $(\tilde x_{1:T}^k, w_T^k)_{k = 1}^K$ at the last time step, we can approximate the posterior as $\sum_{k = 1}^K \bar{w}_T^k \delta_{\tilde{x}_{1:T}^k}(x_{1:T})$ and the integral $I_{\varphi}$ as $\sum_{k = 1}^K \bar{w}_T^k \varphi(\tilde{x}_{1:T}^k)$, where $\bar{w}_T^k := w_T^k / \sum_j w_T^j$ is the normalized weight and $\delta_z$ is a Dirac measure centered on $z$.
Furthermore, one can obtain an unbiased estimator of the marginal likelihood $p_{\theta}(y_{1:T})$ using the intermediate particle weights:
\begin{align}
  \hat Z_{\text{SMC}} := \prod_{t = 1}^T \left[\frac{1}{K} \sum_{k = 1}^K w_t^k \right]. \label{background:Z_SMC}
\end{align}
\begin{algorithm}[ht!]
  \KwData{observed values $y_{1:T}$, model parameters $\theta$, proposal parameters $\phi$}
  \Begin{
    Sample initial particle values $x_1^k \sim q_{1, \phi}(\cdot \given y_1)$. \\
    Compute and normalize weights:
    \begin{align*}
      w_1^k = \frac{\mu_{\theta}(x_1^k) g_{1, \theta}(y_1 \given x_1^k)}{q_{1, \phi}(x_1^k \given y_1)}, && \bar{w}_1^k = \frac{w_1^k}{\sum_{\ell = 1}^K w_1^\ell}.
    \end{align*}
        Initialize particle set: $\tilde{x}_1^k \leftarrow x_1^k$ \\
    \For{$t = 2, 3, \dotsc, T$}{
      Sample ancestor index $a_{t - 1}^k \sim \mathrm{Discrete}(\cdot \given \bar{w}_{t - 1}^1, \dotsc, \bar{w}_{t - 1}^K)$. \\
      Sample particle value $x_t^k \sim q_{t, \phi}(\cdot \given y_{1:t}, \tilde{x}_{1:t - 1}^{a_{t - 1}^k})$. \\
      Update particle set $\tilde{x}_{1:t}^k \leftarrow (\tilde{x}_{1:t - 1}^{a_{t - 1}^k}, x_t^k)$.\\
      Compute and normalize weights:
      \begin{align*}
        w_t^k = \frac{f_{t, \theta}(x_t^k \given {\tilde{x}}_{1:t - 1}^{a_{t - 1}^k}) g_{t, \theta}(y_t \given {\tilde{x}}_{1:t}^k)}{q_{t, \phi}(x_t^k \given y_{1:t}, {\tilde{x}}_{1:t - 1}^{a_{t - 1}^k})}, && \bar{w}_t^k = \frac{w_t^k}{\sum_{\ell = 1}^K w_t^\ell}.
      \end{align*}
      \vspace{-5pt}
    }
    Compute marginal likelihood: $\hat Z_{\text{SMC}} = \prod_{t=1}^{T} \frac{1}{K} \sum_{k = 1}^K w_t^k$.
  }
  \Return{particles $(\tilde{x}_{1:T}^k)_{k=1}^K$, weights $(w_T^k)_{k=1}^K$, marginal likelihood estimate $\hat Z_{\text{SMC}}$}
  \caption{\Glsdesc{SMC}}
  \label{alg:background/smc}
\end{algorithm}

The sequential nature of \gls{SMC} and the resampling step are crucial in making \gls{SMC} scalable to large $T$.
The former makes it easier to design efficient proposal distributions as each step need only target the next set of variables $x_t$. 
The resampling step allows the algorithm to focus on promising particles in light of new observations, avoiding the exponential divergence between the weights of different samples that occurs for importance sampling as $T$ increases.
This can be demonstrated both empirically and theoretically \citep[Chapter 9]{del2004feynman}.
We refer the reader to \citep{doucet2009tutorial} for an in-depth treatment of \gls{SMC}.

\subsection{Importance Weighted Auto-Encoders}
Given a dataset of observations $(y^{(n)})_{n = 1}^N$, a generative network $p_{\theta}(x, y)$ and an inference network $q_{\phi}(x \given y)$, \glspl{IWAE}~\citep{burda2016importance} maximize $\frac{1}{N} \sum_{n = 1}^N \ELBO_{\text{IS}}(\theta, \phi, y^{(n)})$ where, for a given observation $y$, the $\ELBO_{\text{IS}}$ (with $K$ particles) is a lower bound on $\log p_{\theta}(y)$ by Jensen's inequality:
\begin{align}
    &\ELBO_{\text{IS}}(\theta, \phi, y) = \int Q_{\text{IS}}(x^{1:K}) \log \hat Z_{\text{IS}}(x^{1:K}) \,\mathrm dx^{1:K} 
    \leq \log p_{\theta}(y), \text{ where } \label{eq:background/elbo-is} \\
    &Q_{\text{IS}}(x^{1:K}) = \prod_{k = 1}^K q_{\phi}(x^k \given y), 
    \,\,\,\, \hat Z_{\text{IS}}(x^{1:K}) = \frac{1}{K} \sum_{k = 1}^K \frac{p_{\theta}(x^k, y)}{q_{\phi}(x^k \given y)}. \label{eq:background/q_is_z_is}
\end{align}
Note that for $K = 1$ particle, this objective reduces to a \gls{VAE}~\citep{kingma2014auto,rezende2014stochastic} objective we will refer to as
\begin{align}
  \ELBO_{\text{VAE}}(\theta, \phi, y) = \int q_{\phi}(x \given y) (\log p_{\theta}(x, y) - \log q_{\phi}(x \given y)) \,\mathrm dx. \label{eqn:background/elbo_vae}
\end{align}

The \gls{IWAE} optimization is performed using \gls{SGA} where a sample from $\left(\prod_{k = 1}^K q_{\phi}(x^k \given y^{(n)})\right)$ is obtained using the reparameterization trick~\citep{kingma2014auto} and the gradient 
$\frac{1}{N}\sum_{n=1}^{N}\nabla_{\theta, \phi} \log \left(\sum_{k = 1}^K \frac{p_{\theta}(x^k, y^{(n)})}{q_{\phi}(x^k \given y^{(n)})}\right)$ is used to perform an optimization step.

\section{Auto-Encoding Sequential Monte Carlo}

\gls{AESMC} implements model learning, proposal adaptation, and inference
amortization in a similar
manner to the \gls{VAE} and the \gls{IWAE}:
it uses \gls{SGA} on an empirical average of the
\gls{ELBO} over observations.  However, it
varies in the form of this \gls{ELBO}.
In this section, we will introduce the \gls{AESMC} \gls{ELBO}, explain how gradients of it
can be estimated, and discuss the implications of these changes.

\subsection{Objective Function}
Consider a family of \glspl{SSM} $\{p_{\theta}(x_{1:T}, y_{1:T}): \theta \in \Theta\}$ and a family of proposal distributions $\{q_{\phi}(x_{1:T} \given y_{1:T}) = q_{1, \phi}(x_1 \given y_1) \prod_{t = 2}^T q_{t, \phi}(x_t \given x_{1:t - 1}, y_{1:t}): \phi \in \Phi\}$.
\Gls{AESMC} uses an \gls{ELBO} objective based on the \gls{SMC} marginal likelihood estimator \eqref{background:Z_SMC}.
In particular, for a given $y_{1:T}$, the objective is defined as
\begin{align}
  \ELBO_{\text{SMC}}(\theta, \phi, y_{1:T}) &:= \int Q_{\text{SMC}}(x_{1:T}^{1:K}, a_{1:T - 1}^{1:K}) \log \hat Z_{\text{SMC}}(x_{1:T}^{1:K}, a_{1:T - 1}^{1:K}) \,\mathrm dx_{1:T}^{1:K} \,\mathrm da_{1:T - 1}^{1:K},
\end{align}
where $\hat Z_{\text{SMC}}(x_{1:T}^{1:K}, a_{1:T - 1}^{1:K})$ is defined in \eqref{background:Z_SMC} and $Q_{\text{SMC}}$ is the sampling distribution of \gls{SMC},
\begin{align}
  Q_{\text{SMC}}(x_{1:T}^{1:K}, a_{1:T - 1}^{1:K}) &= \left(\prod_{k = 1}^K q_{1, \phi}(x_1^k)\right) \left( \prod_{t = 2}^T \prod_{k = 1}^K q_{t, \phi}(x_t^k \given \tilde x_{1:t - 1}^{a_{t - 1}^k}) \cdot \mathrm{Discrete}(a_{t - 1}^k \given w_{t - 1}^{1:K})\right). \label{eqn:aesmc/q_smc}
\end{align}

$\ELBO_{\text{SMC}}$ forms a lower bound to the log marginal likelihood $\log p_{\theta}(y_{1:T})$ due to Jensen's inequality
and the unbiasedness of the marginal likelihood estimator.
Hence, given a dataset $(y_{1:T}^{(n)})_{n = 1}^N$, we can perform model learning based on maximizing the lower bound of $\frac{1}{N} \sum_{n = 1}^N \log p_{\theta}(y_{1:T}^{(n)})$ as a surrogate target, namely by maximizing
\begin{align}
  \mathcal J(\theta, \phi) &:= \frac{1}{N} \sum_{n = 1}^N \ELBO_{\text{SMC}}(\theta, \phi, y_{1:T}^{(n)}). \label{eqn:aesmc/objective}
\end{align}
For notational convenience, we will talk about optimizing \glspl{ELBO} in the rest
of this section.  However, we note that the main intended use of \gls{AESMC} is to amortize over datasets, for which the \gls{ELBO} 
is replaced by the dataset average $\mathcal J(\theta, \phi)$ in the optimization target.  Nonetheless, rather than using the full dataset 
for each gradient update, will we instead use minibatches, noting that this forms unbiased estimator.

\subsection{Gradient Estimation}

We describe a gradient estimator used for optimizing $\ELBO_{\text{SMC}}(\theta, \phi, y_{1:T})$ using \gls{SGA}.
The \gls{SMC} sampler in Algorithm~\ref{alg:background/smc} proceeds by sampling
$x_1^{1:K}, a_1^{1:K}, x_2^{1:K}, \dotsc$ sequentially from their respective
distributions $\prod_{k = 1}^K q_1(x_1^k), \; \prod_{k = 1}^K
\mathrm{Discrete}(a_1^k \given w_1^{1:K}), \; \prod_{k = 1}^K q_2(x_2^k \given
x_1^{a_1^k}), \dotsc$ until the whole particle-weight trajectory
$(x_{1:K}^{1:T}, a_{1:T - 1}^{1:K})$ is sampled.
From this trajectory, using equation \eqref{background:Z_SMC}, we can obtain an estimator for the
marginal likelihood.

Assuming that the sampling of latent variables $x_{1:T}^{1:K}$ is reparameterizable, we can make their sampling independent of $(\theta, \phi)$.
In particular, assume that there exists a set of auxiliary random variables
$\epsilon_{1:T}^{1:K}$ where $\epsilon_t^k \sim s_t$ and a set of
reparameterization functions $r_t$.
We can simulate the \gls{SMC} sampler by first sampling $\epsilon_1^{1:K} \sim \prod_{k = 1}^K s_1$ and setting $x_1^k = r_1(\epsilon_1^k)$
and $\tilde{x}_1^k = x_1^k$, then
for $t=2,\dots,T$ cycling through
sampling $a_{t-1}^{1:K} \sim \prod_{k = 1}^K \mathrm{Discrete}(a_{t-1}^k \given w_{t-1}^{1:K})$ and
$\epsilon_{t}^{1:K} \sim \prod_{k = 1}^K s_{t}$, and setting $x_{t}^k = r_{t}(\epsilon_{t}^k, \tilde{x}_{1:t-1}^{a_{t-1}^k})$ and $\tilde{x}_{1:t}^k = (\tilde{x}_{1:t - 1}^{a_{t - 1}^k}, x_t^k)$.
We use the resulting reparameterized sample of $(x_{1:K}^{1:T}, a_{1:T - 1}^{1:K})$ to evaluate the
gradient estimator $\nabla_{\theta, \phi} \log \hat Z_{\text{SMC}}(x_{1:T}^{1:K}, a_{1:T - 1}^{1:K})$.

To account for the discrete choices of ancestor indices $a_t^k$ one could additionally use the \gls{REINFORCE}~\citep{williams1992simple} trick, however in practice, we found that the additional term in the estimator has problematically high variance. We explore various other possible gradient estimators and empirical assessments of their variances in Appendix~\ref{sec:gradients}.
This exploration confirms that including the additional \gls{REINFORCE} terms leads to problematically
high variance, justifying our decision to omit them, despite introducing a small bias into the gradient
estimates.

\subsection{Bias \& Implications on the Proposals}

In this section, we express the gap between \glspl{ELBO} and the log marginal likelihood as a \gls{KL} divergence and study implications on the proposal distributions.
We present a set of claims and propositions whose full proofs are in Appendix~\ref{sec:bias-proofs}.  These give insight into the behavior of \gls{AESMC}
and show the advantages, and disadvantages, of using
our different \gls{ELBO}.
This insight motivates Section~\ref{sec:improving_proposal_learning} which proposes an algorithm for improving proposal learning.

\begin{definition}
    \label{claim:aesmc/bias/elbo}
    Given an \emph{unnormalized target density} $\tilde P: \mathcal X \to [0, \infty)$
    with \emph{normalizing constant} $Z_P > 0$, $P := \tilde P / Z_P$,
     and a \emph{proposal density} $Q: \mathcal X \to [0, \infty)$, then
    \begin{align}
        \ELBO := \int Q(x) \log \frac{\tilde P(x)}{Q(x)} \,\mathrm dx, \label{eqn:aesmc/bias/elbo}
    \end{align}
    is a lower bound on $\log Z_P$ and satisfies
    \begin{align}
        \ELBO = \log Z_P - \KL{Q}{P}. \label{eqn:aesmc/bias/elbo2}
    \end{align}
\end{definition}
This is a standard identity used in variational inference and \glspl{VAE}.
In the case of \glspl{VAE}, applying Definition~\ref{claim:aesmc/bias/elbo} with $P$ being $p_{\theta}(x \given y)$, $\tilde{P}$ being $p_{\theta}(x, y)$, $Z_P$ being $p_{\theta}(y)$, and $Q$ being $q_{\phi}(x \given y)$, we can directly
rewrite \eqref{eqn:background/elbo_vae} as $\ELBO_{\text{VAE}}(\theta, \phi, y) = \log p_{\theta}(y) - \KL{q_{\phi}(x \given y)}{p_{\theta}(x \given y)}$.

The key observation for expressing such a bound for general \glspl{ELBO} such as $\ELBO_{\text{IS}}$ and $\ELBO_{\text{SMC}}$ is that the target density $P$ and the proposal density $Q$ need not directly correspond to $p_{\theta}(x \given y)$ and $q_{\phi}(x \given y)$.
This allows us to view the underlying sampling distributions of the marginal likelihood Monte Carlo estimators such as $Q_{\text{IS}}$ in \eqref{eq:background/q_is_z_is} and $Q_{\text{SMC}}$ in \eqref{eqn:aesmc/q_smc} as proposal distributions on an extended space $\mathcal X$.
The following claim uses this observation to express the bound between a general \gls{ELBO} and the log marginal likelihood as $\gls{KL}$ divergence from the extended space sampling distribution to a corresponding target distribution.
\begin{claim}
    \label{claim:aesmc/bias/estimator}
    Given a non-negative unbiased estimator $\hat Z_P(x) \geq 0$ of the normalizing constant $Z_P$ where $x$ is distributed according to the proposal distribution $Q(x)$, the following holds:
    \begin{align}
        \ELBO &= \int Q(x) \log \hat Z_P(x) \,\mathrm dx = \log Z_P - \KL{Q}{P},
        \label{eqn:aesmc/bias/elbo3} \\
        \text{where} \quad P(x) &= \frac{Q(x) \hat Z_P(x)}{Z_P} \label{eqn:aesmc/bias/P}
    \end{align}
    is the implied normalized target density.
\end{claim}

In the case of \glspl{IWAE}, we can apply Claim~\ref{claim:aesmc/bias/estimator} with $Q$ and $\hat Z_P$ 
being $Q_{\text{IS}}$ and $\hat Z_{\text{IS}}$ respectively as defined in \eqref{eq:background/q_is_z_is} and $Z_P$ being $p_{\theta}(y)$.
This yields
\begin{align}
  \ELBO_{\text{IS}}(\theta, \phi, y) &= \log p_{\theta}(y) - \KL{Q_{\text{IS}}}{P_{\text{IS}}}, \text{ where } \label{eqn:aesmc/bias/elbo_is} \\
  P_{\text{IS}}(x^{1:K}) &= \frac{1}{K} \sum_{k = 1}^K \left(q_\phi(x^1 \given y) \cdots q_\phi(x^{k - 1} \given y) p_\theta(x^k \given y) q_\phi(x^{k + 1} \given y) \cdots q_\phi(x^K \given y) \right).
\end{align}
Similarly, in the case of \gls{AESMC}, we obtain
\begin{align}
  \ELBO_{\text{SMC}}(\theta, \phi, y_{1:T}) &= \log p_{\theta}(y_{1:T}) - \KL{Q_{\text{SMC}}}{P_{\text{SMC}}}, \text{ where } \label{eqn:aesmc/bias/elbo_smc} \\
  P_{\text{SMC}}(x_{1:T}^{1:K}, a_{1:T - 1}^{1:K}) &= Q_{\text{SMC}}(x_{1:T}^{1:K}, a_{1:T - 1}^{1:K}) \hat Z_{\text{SMC}}(x_{1:T}^{1:K}, a_{1:T - 1}^{1:K}) / p_{\theta}(y_{1:T}).
\end{align}

Having expressions for the target distribution $P$ and the sampling distribution $Q$ for a given \gls{ELBO} allows us to investigate what happens when we maximize that \gls{ELBO},
remembering that the \gls{KL} term is strictly non-negative  and zero if and only if $P = Q$.
For the \gls{VAE} and \gls{IWAE} cases then, provided the proposal is sufficiently flexible, one
can always perfectly maximize the \gls{ELBO} by setting
$p_{\theta}(x \given y) = q_{\phi}(x \given y)$ for all $x$.
The reverse implication also holds: if $\ELBO_{\text{VAE}}=\log Z_P$ then it must be the case that $p_{\theta}(x \given y) = q_{\phi}(x \given y)$.
However, for \gls{AESMC}, achieving $\ELBO=\log Z_P$ is only possible
when one
also has sufficient flexibility to learn a particular series of intermediate target distributions,
namely the marginals of the final target distribution.  In other words, it is necessary to
learn a particular factorization of the generative model, not just the correct individual
proposals,
to achieve $P=Q$ and thus $\ELBO_{\text{SMC}} = Z_P$.
These observations are formalized in Propositions~\ref{proposition:aesmc/bias/elbo_is} and \ref{proposition:aesmc/bias/elbo_smc} below.
\begin{proposition}
    \label{proposition:aesmc/bias/elbo_is}
    $Q_{\text{IS}}(x^{1:K}) = P_{\text{IS}}(x^{1:K})$ for all $x^{1:K}$ if and only if $q(x \given y) = p(x \given y)$ for all $x$.
\end{proposition}
\begin{proposition}
    \label{proposition:aesmc/bias/elbo_smc}
    If $K > 1$, then $P_{\text{SMC}}(x_{1:T}^{1:K}, a_{1:T - 1}^{1:K}) = Q_{\text{SMC}}(x_{1:T}^{1:K}, a_{1:T - 1}^{1:K})$ for all $(x_{1:T}^{1:K}, a_{1:T - 1}^{1:K})$ if and only if
    \begin{enumerate}
        \item $\pi_t (x_{1:t})= \int p(x_{1:T} \given y_{1:T}) \,\mathrm dx_{t + 1:T} = p(x_{1:t} \given y_{1:T})$ for all $x_{1:t}$ and $t = 1, \dotsc, T$, and
        \item $q_1(x_1 \given y_1) = p(x_1 \given y_{1:T})$ for all $x_1$ and $q_t(x_t \given x_{1:t - 1}, y_{1:t}) = p(x_{1:t} \given y_{1:T}) / p(x_{1:t - 1} \given y_{1:T})$ for $t = 2, \dotsc, T$ for all $x_{1:t}$,
    \end{enumerate}
    where $\pi_t (x_{1:t})$ are the intermediate targets used by \gls{SMC}.
\end{proposition}
Proposition~\ref{proposition:aesmc/bias/elbo_smc} has the consequence that if the family of generative models is such that
the first condition does not hold, we will not be able to
make the bound tight.  This means that,
except for a very small class of models,  then, for most
 convenient parameterizations, it will be impossible to learn a perfect proposal that
 gives a tight bound, i.e. there will be no $\theta$ and $\phi$ such that the above conditions
 can be satisfied.  However, it also means that $\ELBO_{\text{SMC}}$ encodes important
 additional information about the implications the factorization of the generative model
 has on the inference---the model depends only on
 the final target $\pi_T(x_{1:T}) = p_{\theta}(x_{1:T} | y_{1:T})$, but some choices of
 the intermediate targets $\pi_t(x_{1:t})$ will lead to much more efficient inference than
 others.
 Perhaps more importantly, \gls{SMC} is usually a
 far more powerful inference algorithm than importance sampling and so the \gls{AESMC}
 setup allows for more ambitious model learning problems to be effectively tackled than
 the \gls{VAE} or \gls{IWAE}.  After all, even though it is well known in the \gls{SMC} literature
 that, unlike
 for \gls{IS}, most problems have no perfect set of \gls{SMC} proposals which
 will generate exact samples from
 the posterior~\citep{doucet2009tutorial}, \gls{SMC} still gives superior performance
 on most problems with more than a few dimensions.
These intuitions are backed up by our experiments that show that using $\ELBO_{\text{SMC}}$
regularly learns better models than using $\ELBO_{\text{IS}}$.

\section{Improving Proposal Learning}
\label{sec:improving_proposal_learning}

In practice, one is rarely able to perfectly drive the divergence to zero and achieve a perfect proposal.
In addition to the implications of the previous section, this occurs because $q_{\phi}(x_{1:T} \given y_{1:T})$
may not be sufficiently expressive to represent $p_{\theta}(x_{1:T} \given y_{1:T})$ exactly and
because of the inevitable sub-optimality of the optimization process, remembering that we are aiming to
learn an amortized inference artifact, rather than a single posterior representation.  Consequently, 
to accurately assess the merits of different~\glspl{ELBO} for proposal learning, 
it is necessary to consider their finite-time performance.  We therefore now consider the effect the
number of particles $K$ has on the gradient estimators for  $\ELBO_{\text{IS}}$ and $\ELBO_{\text{SMC}}$.

\begin{wrapfigure}{r}{0.24\textwidth}
    \centering
    \vspace{-9pt}
        \includegraphics[width=0.23\textwidth,trim={0 0 0 0cm},clip]{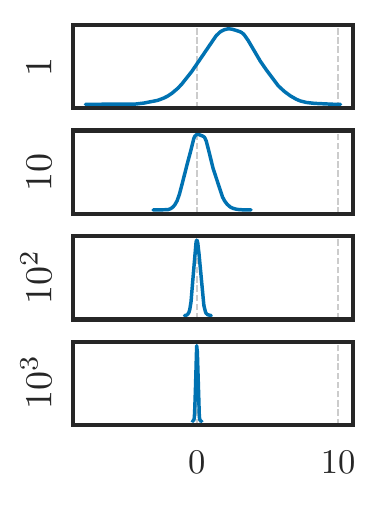}
        \vspace{-15pt}
    \caption{Density estimate of $\nabla_{\phi} \ELBO$
        for different $K$ \label{fig:kde}}
    \vspace{-12pt}
\end{wrapfigure}
Counter-intuitively, it transpires that the tighter bounds implied by using a larger
$K$ is often harmful to proposal learning for both \gls{IWAE} and \gls{AESMC}.
At a high-level, this is because an accurate estimate for $\hat{Z}_P$ can be achieved
for a wide range of proposal parameters $\phi$ and so the magnitude of
$\nabla_{\phi} \ELBO$ reduces as $K$ increases.  Typically, this shrinkage happens faster than
increasing $K$ reduces the standard deviation of the estimate and so the standard deviation of the 
gradient estimate relative to the problem scaling (i.e. as a ratio of true gradient $\nabla_{\phi} \ELBO$) actually increases.
This effect is demonstrated in Figure~\ref{fig:kde}
which shows a kernel density estimator for the distribution of the gradient
estimate for different $K$ and the model given in Section~\ref{sec:experiments/proposal_learning}.  Here
we see that as we increase $K$, both the expected gradient estimate (which is equal to the true gradient by unbiasedness)
and standard deviation of the estimate decrease.  However, the
former decreases faster and so the relative standard deviation increases.
 This is perhaps easiest
to appreciate by noting that for $K > 10$, there is a roughly equal probability
of the estimate being positive or negative, such that we are equally likely to increase or decrease
the parameter value at the next \gls{SGA} iteration, inevitably leading to poor performance.
On the other hand, when $K=1$, it is far more likely that the gradient estimate is positive
than negative, and so there is clear drift to the gradient steps.  We add to the
empirical evidence for this behavior in Section~\ref{sec:exp}. Note the critical
difference for model learning is that $\nabla_{\theta} \ELBO$ does not,
in general, decrease in magnitude as $K$ increases.  Note also that using a larger
$K$ should always give better performance at test time; it may though
be better to learn $\phi$ using a smaller $K$.

In simultaneously developed work~\citep{rainforth2017tighter}, we formalized this intuition in
the \gls{IWAE} setting by showing
that the estimator of $\nabla_\phi \ELBO_{\text{IS}}(\theta, \phi, x)$ with $K$ particles, denoted by $I_K$, has the following \gls{SNR}:
\begin{align}
  \SNR := \frac{\E[I_K]}{\sqrt{\Var[I_K]}} = O\left(\sqrt{\frac{1}{K}}\right).
\end{align}
We thus see that increasing $K$ reduces the \gls{SNR} and so the gradient updates for the proposal
will degrade towards pure noise if $K$ is set too high.

\subsection{Alternating ELBOs}
\label{sec:aesmc/alt}

To address these issues, we suggest and investigate the \gls{ALT} algorithm
which updates $(\theta, \phi)$ in a coordinate descent
fashion using different \glspl{ELBO}, and thus gradient estimates, for each.
We pick a $\theta$-optimizing pair and a $\phi$-optimizing pair $(A_{\theta}, K_{\theta}), (A_{\phi}, K_{\phi}) \in \{\gls{IS}, \gls{SMC}\} \times \{1, 2, \dotsc\}$,
corresponding to an inference type and number of particles.
In an optimization step, we obtain an estimator for $\nabla_{\theta} \ELBO_{A_{\theta}}$ with $K_{\theta}$ particles and an estimator for $\nabla_{\phi} \ELBO_{A_{\phi}}$ with $K_{\phi}$ particles which we call $g_{\theta}$ and $g_{\phi}$ respectively.
We use $g_{\theta}$ to update the current $\theta$ and $g_{\phi}$ to update the current $\phi$.
The results from the previous sections suggest that using $A_{\theta} = \gls{SMC}$ and $A_{\phi} = \gls{IS}$ with a large $K_{\theta}$ and a small $K_{\phi}$ may perform better model and proposal learning than just fixing $(A_{\theta}, K_{\theta}) = (A_{\phi}, K_{\phi})$ to $(\gls{SMC}, \text{large})$ since using $A_{\phi} = \gls{IS}$ with small $K_{\phi}$ helps learning $\phi$ (at least in terms of the~\gls{SNR})
and using $A_{\theta} = \gls{SMC}$ with large $K_{\theta}$ helps learning $\theta$.
We experimentally observe that this procedure can in some cases improve both model and proposal learning.

\section{Experiments}
\label{sec:exp}

We now present a series of experiments designed to answer the following questions:
1) Does tightening the bound by using either more particles or a better inference procedure lead to an adverse 
effect on proposal learning?
2) Can \gls{AESMC}, despite this effect, outperform \gls{IWAE}?
3) Can we further improve the learned model and proposal by using \gls{ALT}?

First we investigate a \gls{LGSSM} for model learning and a latent variable model for proposal adaptation. This allows us to compare the learned parameters to the optimal ones. Doing so, we confirm our conclusions for this simple problem.

We then extend those results to more complex, high dimensional observation spaces that require models and proposals parameterized by neural networks. We do so by investigating the \emph{Moving Agents} dataset, a set of partially occluded video sequences.

\subsection{Linear Gaussian State Space Model}
\label{sec:experiments/lgssm}

Given the following \gls{LGSSM}
\begin{align}
  p(x_1) &= \mathrm{Normal}\left(x_1; 0, 1^2\right), \\
  p(x_t \given x_{t - 1}) &= \mathrm{Normal}\left(x_t; \theta_1 x_{t - 1}, 1^2\right), && t = 2, \dotsc T, \\
  p(y_t \given x_t) &= \mathrm{Normal}\left(y_t; \theta_2 x_t, \sqrt{0.1}^2\right), && t = 1, \dotsc, T,
\end{align}
we find that optimizing $\ELBO_{\text{SMC}}(\theta, \phi, y_{1:T})$ w.r.t. $\theta$ leads to better generative models than optimizing $\ELBO_{\text{IS}}(\theta, \phi, y_{1:T})$. The same is true for using more particles.

We generate a sequence $y_{1:T}$ for $T = 200$ by sampling from the model with $\theta = (\theta_1, \theta_2) = (0.9, 1.0)$.
We then optimize the different $\gls{ELBO}$s w.r.t. $\theta$ using the bootstrap proposal $q_1(x_1 \given y_1) = \mu_{\theta}(x_1)$ and $q_t(x_t \given x_{1:t - 1}, y_{1:t}) = f_{t, \theta}(x_t \given x_{1:t - 1})$.
Because we use the bootstrap proposal, gradients w.r.t. to $\theta$ are not
backpropagated through $q$.

We use a fixed learning rate of $0.01$ and optimize for $500$ steps using \gls{SGA}.
Figure~\ref{fig:experiments/lgssm/logz_and_params} shows that the convergence of both $\log p_{\theta}(y_{1:T})$ to $\max_{\theta} \log p_{\theta}(y_{1:T})$ and $\theta$ to $\argmax_{\theta} \log p_{\theta}(y_{1:T})$ is faster when $\ELBO_{\text{SMC}}$ and more particles are used.
\begin{figure}[!htb]
\centering
\begin{subfigure}{.49\textwidth}
  \centering
  \includegraphics{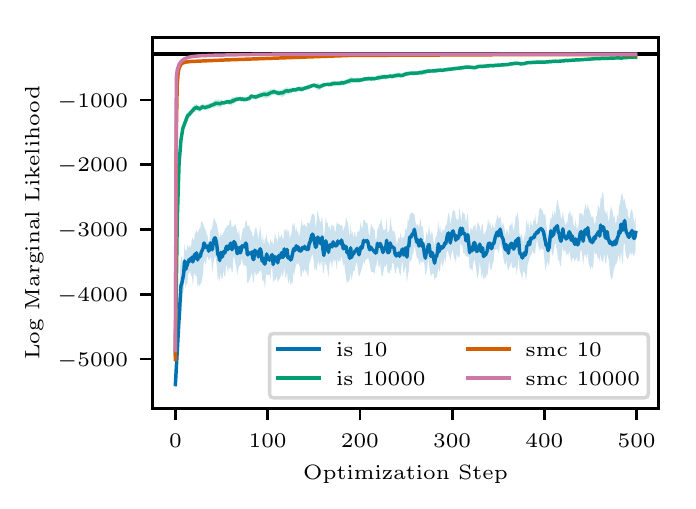}
\end{subfigure}
\begin{subfigure}{.49\textwidth}
  \centering
  \includegraphics{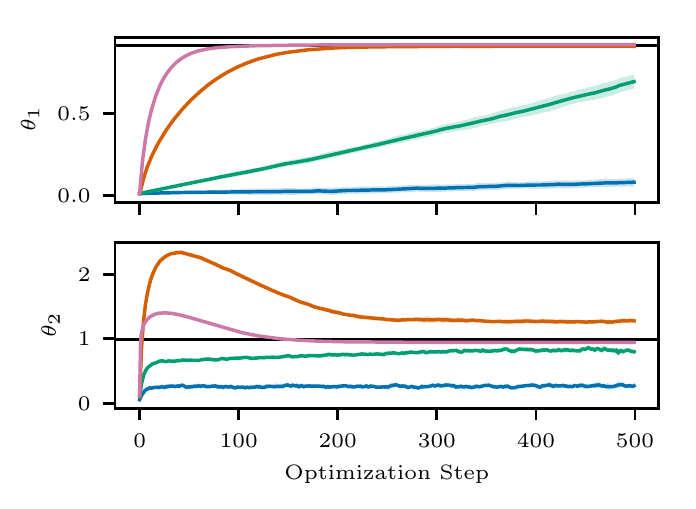}
\end{subfigure}
\caption{(Left) Log marginal likelihood analytically evaluated at every $\theta$ during optimization; the black line indicates $\max_{\theta} \log p_{\theta}(y_{1:T})$ obtained by the \gls{EM} algorithm. (Right) learning of model parameters; the black line indicates $\argmax_{\theta} \log p_{\theta}(y_{1:T})$ obtained by the \gls{EM} algorithm.}
\label{fig:experiments/lgssm/logz_and_params}
\end{figure}

\subsection{Proposal Learning}
We now investigate how learning $\phi$, i.e. the proposal, is affected by the the choice of \gls{ELBO} and the number of particles.

\label{sec:experiments/proposal_learning}
Consider a simple, fixed generative model $p(\mu)p(x \given \mu) = \mathrm{Normal}(\mu; 0, 1^2)\mathrm{Normal}(x; \mu, 1^2)$ where $\mu$ and $x$ are the latent and observed variables respectively and a family of proposal distributions $q_{\phi}(\mu) = \mathrm{Normal}(\mu; \mu_q, \sigma_q^2)$ parameterized by $\phi = (\mu_q, \log \sigma_q^2)$.
For a fixed observation $x = 2.3$, we initialize $\phi = (0.01, 0.01)$ and optimize $\ELBO_{\text{IS}}$ with respect to $\phi$.
We investigate the quality of the learned parameter $\phi$ as we increase the number of particles $K$ during training.
Figure~\ref{fig:experiments/proposals/lgssm}~(left) clearly demonstrates that the quality of $\phi$ compared to the analytic posterior decreases as we increase $K$.

Similar behavior is observed in Figure~\ref{fig:experiments/proposals/lgssm}~(middle, right) where we optimize $\ELBO_{\text{SMC}}$ with respect to both $\theta$ and $\phi$ for the \gls{LGSSM} described in Section~\ref{sec:experiments/lgssm}.
We see that using more particles helps model learning but makes proposal learning worse.
Using our \gls{ALT} algorithm alleviates this problem and at the same time makes model learning faster as it profits from a more accurate proposal distribution.
We provide more extensive experiments exploring proposal learning with different \glspl{ELBO} and number of particles in Appendix~\ref{sec:additional-experiments}.
\begin{figure}[!htb]
\centering
\begin{subfigure}{.32\textwidth}
  \centering
  \includegraphics{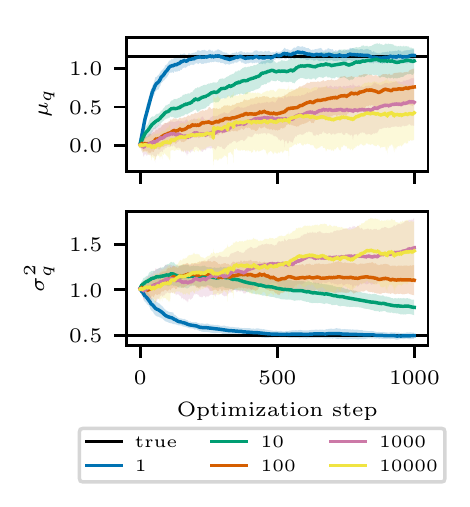}
\end{subfigure}
\begin{subfigure}{.64\textwidth}
  \centering
  \includegraphics{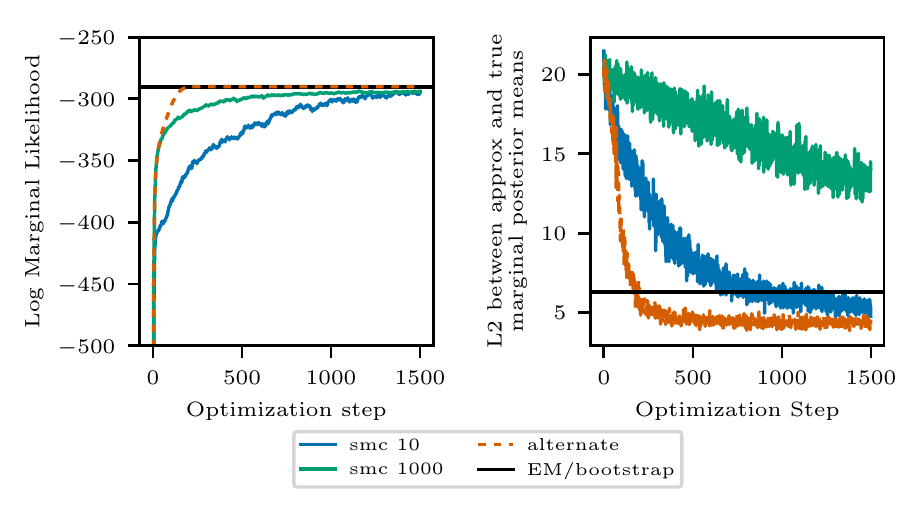}
\end{subfigure}
\caption{\emph{(Left)} Optimizing $\ELBO_{\text{IS}}$ for the Gaussian unknown mean model with respect to $\phi$ results in worse $\phi$ as we increase number of particles $K$. \emph{(Middle, right)} Optimizing $\ELBO_{\text{SMC}}$ with respect to $(\theta, \phi)$ for \gls{LGSSM} and using the \gls{ALT} algorithm for updating $(\theta, \phi)$ with $(A_{\theta}, K_{\theta}) = (\gls{SMC}, 1000)$ and $(A_{\phi}, K_{\phi}) = (\gls{IS}, 10)$. \emph{Right} measures the quality of $\phi$ by showing $\sqrt{\sum_{t = 1}^T (\mu_t^{\text{kalman}} - \mu_t^{\text{approx}})^2}$ where $\mu_t^{\text{kalman}}$ is the marginal mean obtained from the Kalman smoothing algorithm under the model with \gls{EM}-optimized parameters and $\mu_t^{\text{approx}}$ is an marginal mean obtained from the set of $10$ \gls{SMC} particles with learned/bootstrap proposal.}
\label{fig:experiments/proposals/lgssm}
\end{figure}

\subsection{Moving Agents}
\label{sec:moving-agents}

To show that our results are applicable to complex, high dimensional data we compare \gls{AESMC} and \gls{IWAE} on stochastic, partially observable video sequences.
Figure~\ref{fig:experiments/max/visualisation} in Appendix~\ref{sec:appendix_moving_agents} shows an example of such a sequence.

The dataset consists of $N = 5000$ sequences of images $(y_{1:T}^{(n)})_{n = 1}^{N}$ of which 1000 are randomly held out as test set.
Each sequence contains $T = 40$ images represented as a 2 dimensional array of size $32\times32$.
In each sequence there is one agent, represented as circle, whose starting position is sampled randomly along the top and bottom of the image.
The dataset is inspired by \citep{ondruska2016deep}, however with the crucial difference that the movement of the agent is \emph{stochastic}.
The agent performs a directed random walk through the image.
At each timestep, it moves according to
\begin{equation}
  \begin{split}
    y_{t+1} & \sim \mathrm{Normal}(y_{t+1}; y_t+0.15, 0.02^2)\\
    x_{t+1} & \sim \mathrm{Normal}(x_{t+1}; 0, 0.02^2)
  \end{split}
\end{equation}
where $(x_t, y_t)$ are the coordinates in frame $t$ in a unit square that is then projected onto $32\times32$ pixels.
In addition to the stochasticity of the movement, half of the image is occluded, preventing the agent from being observed.

For the generative model and proposal distribution we use a \gls{VRNN} \citep{chung2015recurrent}.
It extends \glspl{RNN} by introducing a stochastic latent state $x_t$ at each timestep $t$. Together with the observation $y_t$, this state conditions the deterministic transition of the \gls{RNN}. By introducing this unobserved stochastic state, the \gls{VRNN} is able to better model complex long range variability in stochastic sequences.
Architecture and hyperparameter details are given in Appendix~\ref{sec:appendix_vrnn}.

Figure \ref{fig:experiments/max/elbos} shows
$\max(\ELBO_{\text{IS}},\ELBO_{\text{SMC}})$ for models trained with
\gls{IWAE} and \gls{AESMC} for different particle numbers. The lines correspond
to the mean over three different random seeds and the shaded areas indicate the
standard deviation.
The same number of particles was used for training and testing, additional
hyperparameter settings are given in the appendix.
One can see that models trained using
\gls{AESMC} outperform \gls{IWAE} and using more particles improves the
\gls{ELBO} for both.
In Appendix~\ref{sec:appendix_moving_agents}, we inspect different learned generative models by using them for prediction, confirming the results presented here.
We also tested \gls{ALT} on this task, but found that while it did occasionally
improve performance, it was much less stable than \gls{IWAE} and \gls{AESMC}.

\begin{figure}[!htb]
\centering
\begin{subfigure}{.49\textwidth}
  \centering
  \includegraphics[width=.95\linewidth]{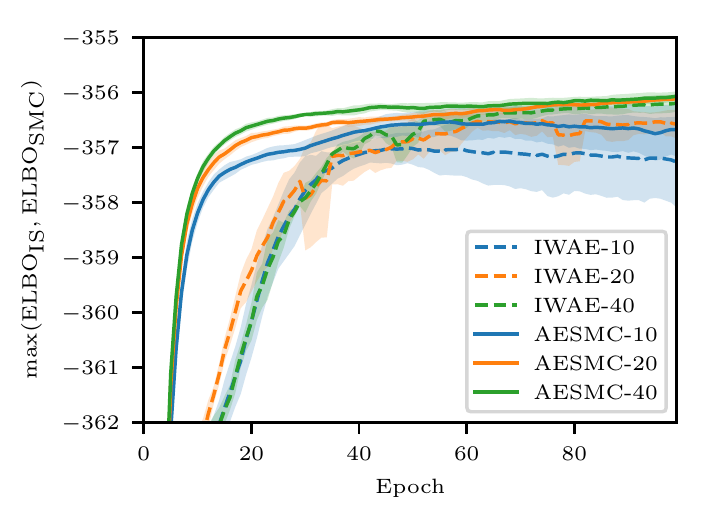}
\end{subfigure}
\begin{subfigure}{.49\textwidth}
  \centering
  \begin{tabular}{llllll}
    \toprule
  Particles           & Method        & Moving Agents   \\ \hline
  \multirow{2}{*}{10}  & \gls{IWAE}  & -357.3          \\
                      & \gls{AESMC} & \textbf{-356.7}          \\\hline
  \multirow{2}{*}{20} & \gls{IWAE}  & -356.6          \\
                      & \gls{AESMC} & \textbf{-356.1}          \\\hline
  \multirow{2}{*}{40} & \gls{IWAE}  & -356.2          \\
                      & \gls{AESMC} & \textbf{-356.1}    \\   
    \bottomrule
  \end{tabular}
\end{subfigure}
\caption{\emph{(Left)} Rolling mean over 5 epochs of $\max(\ELBO_{\text{SMC}},
  \ELBO_{\text{IS}})$ on the test set, lines indicate the average over 3 random
  seeds and shaded areas indicate standard deviation. The color indicates the number of particles, the line style the used algorithm. \emph{(Right)} The table shows the final $\max(\ELBO_{\text{SMC}}, \ELBO_{\text{IS}})$ for each learned model.}
\label{fig:experiments/max/elbos}
\end{figure}

\section{Conclusions}

We have developed \gls{AESMC}---a method for performing model learning using a new \gls{ELBO} objective which is based on the \gls{SMC} marginal likelihood estimator.
This \gls{ELBO} objective is optimized using \gls{SGA} and the reparameterization trick.
Our approach utilizes the efficiency of \gls{SMC} in models with intermediate observations and hence is suitable for highly structured models.
We experimentally demonstrated that this objective leads to better generative model training
than the \gls{IWAE} objective for structured problems, due to the superior inference and tighter bound
provided by using \gls{SMC} instead of importance sampling.

Additionally, in Claim~\ref{claim:aesmc/bias/estimator}, we provide a simple way to express the bias of objectives induced by log of marginal likelihood estimators as a \gls{KL} divergence on an extended space.
In Propositions~\ref{proposition:aesmc/bias/elbo_is} and \ref{proposition:aesmc/bias/elbo_smc}, we investigate the implications of these \glspl{KL} being zero in the case of \gls{IWAE} and \gls{AESMC}.
In the latter case, we find that we can achieve zero \gls{KL} only if 
we are able to learn \gls{SMC} intermediate target distributions corresponding to
marginals of the target distribution.  
Using our assertion that 
tighter variational bounds are not necessarily better, we then introduce
and test a new method, alternating~\glspl{ELBO}, that addresses some of these issues and 
observe that, in some cases, this improves both model and proposal learning.

\newpage

\subsubsection*{Acknowledgments}

TAL is supported by EPSRC DTA and Google (project code DF6700) studentships.
MI is supported by the UK EPSRC CDT in Autonomous Intelligent Machines and Systems.
TR is supported by the European Research Council under the European Union’s Seventh Framework Programme (FP7/2007-2013) ERC grant agreement no. 617071; majority of TR's work was undertaken while he was in the Department of Engineering Science, University of Oxford, and was supported by a BP industrial grant.
TJ is supported by the UK EPSRC and MRC CDT in Statistical Science.
FW is supported by The Alan Turing Institute under the EPSRC grant EP/N510129/1; DARPA PPAML through the U.S. AFRL under Cooperative Agreement FA8750-14-2-0006; Intel and DARPA D3M, under Cooperative Agreement FA8750-17-2-0093.

\bibliography{main}
\bibliographystyle{iclr2018_conference}

\appendix

\newpage
\section{Gradients}
\label{sec:gradients}

The goal is to obtain an unbiased estimator for the gradient
\begin{align}
    \nabla_{\theta, \phi} \int Q_{\text{SMC}}(x_{1:T}^{1:K}, a_{1:T - 1}^{1:K}) \log \hat Z_{\text{SMC}}(x_{1:T}^{1:K}, a_{1:T - 1}^{1:K}) \,\mathrm dx_{1:T}^{1:K} \,\mathrm da_{1:T - 1}^{1:K}.
\end{align}

\subsection{Full Reinforce}
We express the required quantity as
\begin{align}
    &\nabla_{\theta, \phi} \int Q_{\text{SMC}}(x_{1:T}^{1:K}, a_{1:T - 1}^{1:K}) \log \hat Z_{\text{SMC}}(x_{1:T}^{1:K}, a_{1:T - 1}^{1:K}) \,\mathrm dx_{1:T}^{1:K} \,\mathrm da_{1:T - 1}^{1:K} \\
    &= \int \nabla_{\theta, \phi} Q_{\text{SMC}}(x_{1:T}^{1:K}, a_{1:T - 1}^{1:K}) \log \hat Z_{\text{SMC}}(x_{1:T}^{1:K}, a_{1:T - 1}^{1:K}) + \\
    & \hspace{3em} Q_{\text{SMC}}(x_{1:T}^{1:K}, a_{1:T - 1}^{1:K}) \nabla_{\theta, \phi} \log \hat Z_{\text{SMC}}(x_{1:T}^{1:K}, a_{1:T - 1}^{1:K}) \,\mathrm dx_{1:T}^{1:K} \,\mathrm da_{1:T - 1}^{1:K} \\
    &= \int Q_{\text{SMC}}(x_{1:T}^{1:K}, a_{1:T - 1}^{1:K}) \left[\nabla_{\theta, \phi} \log Q_{\text{SMC}}(x_{1:T}^{1:K}, a_{1:T - 1}^{1:K}) \log \hat Z_{\text{SMC}}(x_{1:T}^{1:K}, a_{1:T - 1}^{1:K}) + \right. \\
    &\left. \hspace{3em} \nabla_{\theta, \phi} \log \hat Z_{\text{SMC}}(x_{1:T}^{1:K}, a_{1:T - 1}^{1:K})\right]  \,\mathrm dx_{1:T}^{1:K} \,\mathrm da_{1:T - 1}^{1:K},
\end{align}
which we can estimate by sampling $(x_{1:T}^{1:K}, a_{1:T - 1}^{1:K})$ directly from $Q_{\text{SMC}}$ and evaluating $\left[\nabla_{\theta, \phi} \log Q_{\text{SMC}}(x_{1:T}^{1:K}, a_{1:T - 1}^{1:K}) \log \hat Z_{\text{SMC}}(x_{1:T}^{1:K}, a_{1:T - 1}^{1:K}) + \nabla_{\theta, \phi} \log \hat Z_{\text{SMC}}(x_{1:T}^{1:K}, a_{1:T - 1}^{1:K})\right]$.

\subsection{Reinforce \& Reparameterization}
\label{sec:gradients/reinforce-reparam}

We express the required quantity as
\begin{align}
    \nabla_{\theta, \phi} &\int Q_{\text{SMC}}(x_{1:T}^{1:K}, a_{1:T - 1}^{1:K}) \log \hat Z_{\text{SMC}}(x_{1:T}^{1:K}, a_{1:T - 1}^{1:K}) \,\mathrm dx_{1:T}^{1:K} \,\mathrm da_{1:T - 1}^{1:K} \\
    &= \nabla_{\theta, \phi} \int \left(\prod_{k = 1}^K q_1(x_1^k)\right) \left( \prod_{t = 2}^T \prod_{k = 1}^K q_t(x_t^k \given x_{t - 1}^{a_{t - 1}^k}) \cdot \mathrm{Discrete}(a_{t - 1}^k \given w_{t - 1}^{1:K})\right) \nonumber\\
    & \hspace{3em} \log \hat Z_{\text{SMC}}(x_{1:T}^{1:K}, a_{1:T - 1}^{1:K}) \,\mathrm dx_{1:T}^{1:K} \,\mathrm da_{1:T - 1}^{1:K} \\
    &= \nabla_{\theta, \phi} \int \left(\prod_{k = 1}^K s_1(\epsilon_1^k)\right) \left( \prod_{t = 2}^T \prod_{k = 1}^K s_t(\epsilon_t^k) \cdot \mathrm{Discrete}(a_{t - 1}^k \given w_{t - 1}^{1:K})\right) \nonumber\\
    & \hspace{3em} \log \hat Z_{\text{SMC}}(r(\epsilon_{1:T}^{1:K}), a_{1:T - 1}^{1:K}) \,\mathrm d\epsilon_{1:T}^{1:K} \,\mathrm da_{1:T - 1}^{1:K} \\
    &= \int \left( \prod_{t = 1}^T \prod_{k = 1}^K s_t(\epsilon_t^k) \right) \left[\nabla_{\theta, \phi} \prod_{t = 2}^T \prod_{k = 1}^K \mathrm{Discrete}(a_{t - 1}^k \given w_{t - 1}^{1:K}) \log \hat Z_{\text{SMC}}(r(\epsilon_{1:T}^{1:K}), a_{1:T - 1}^{1:K}) + \right. \nonumber\\
    &\hspace{3em} \left. \left( \prod_{t = 2}^T \prod_{k = 1}^K \mathrm{Discrete}(a_{t - 1}^k \given w_{t - 1}^{1:K}) \right) \nabla_{\theta, \phi} \log \hat Z_{\text{SMC}}(r(\epsilon_{1:T}^{1:K}), a_{1:T - 1}^{1:K}) \right] \,\mathrm d\epsilon_{1:T}^{1:K} \,\mathrm da_{1:T - 1}^{1:K}  \\
    &= \int \left( \prod_{t = 1}^T \prod_{k = 1}^K s_t(\epsilon_t^k) \right) \left( \prod_{t = 2}^T \prod_{k = 1}^K \mathrm{Discrete}(a_{t - 1}^k \given w_{t - 1}^{1:K}) \right) \cdot \nonumber\\
    & \hspace{3em} \left[
        \nabla_{\theta, \phi} \log\left(
            \prod_{t = 2}^T \prod_{k = 1}^K \mathrm{Discrete}(a_{t - 1}^k \given w_{t - 1}^{1:K})
        \right)
        \log \hat Z_{\text{SMC}}(r(\epsilon_{1:T}^{1:K}), a_{1:T - 1}^{1:K}) + \right.  \nonumber\\
    & \hspace{3em} \left.\nabla_{\theta, \phi} \log \hat Z_{\text{SMC}}(r(\epsilon_{1:T}^{1:K}), a_{1:T - 1}^{1:K}) \label{eq:gradients/reinforce-reparam}
    \right] \,\mathrm d\epsilon_{1:T}^{1:K} \,\mathrm da_{1:T - 1}^{1:K},
\end{align}
where $r(\epsilon_{1:T}^{1:K})$ denotes a sample with identical distribution as $x_{1:T}^{1:K}$ obtained by passing the auxiliary samples $\epsilon_{1:T}^{1:K}$ through the reparameterization function.
We can thus estimate the gradient by sampling $\epsilon_{1:T}^{1:K}$ from the auxiliary distribution, reparameterizing and evaluating $\left[
    \nabla_{\theta, \phi} \log\left(
        \prod_{t = 2}^T \prod_{k = 1}^K \mathrm{Discrete}(a_{t - 1}^k \given w_{t - 1}^{1:K})
    \right)
    \log \hat Z_{\text{SMC}}(r(\epsilon_{1:T}^{1:K}), a_{1:T - 1}^{1:K}) + \nabla_{\theta, \phi} \log \hat Z_{\text{SMC}}(r(\epsilon_{1:T}^{1:K}), a_{1:T - 1}^{1:K})
\right]$.

In Figure~\ref{fig:gradients/reinforce_vs_ignore}, we demonstrate that the estimator in \eqref{eq:gradients/reinforce-reparam} has much higher variance if we include the first term.
\begin{figure}[!htb]
  \centering
  \includegraphics[width=0.5\textwidth]{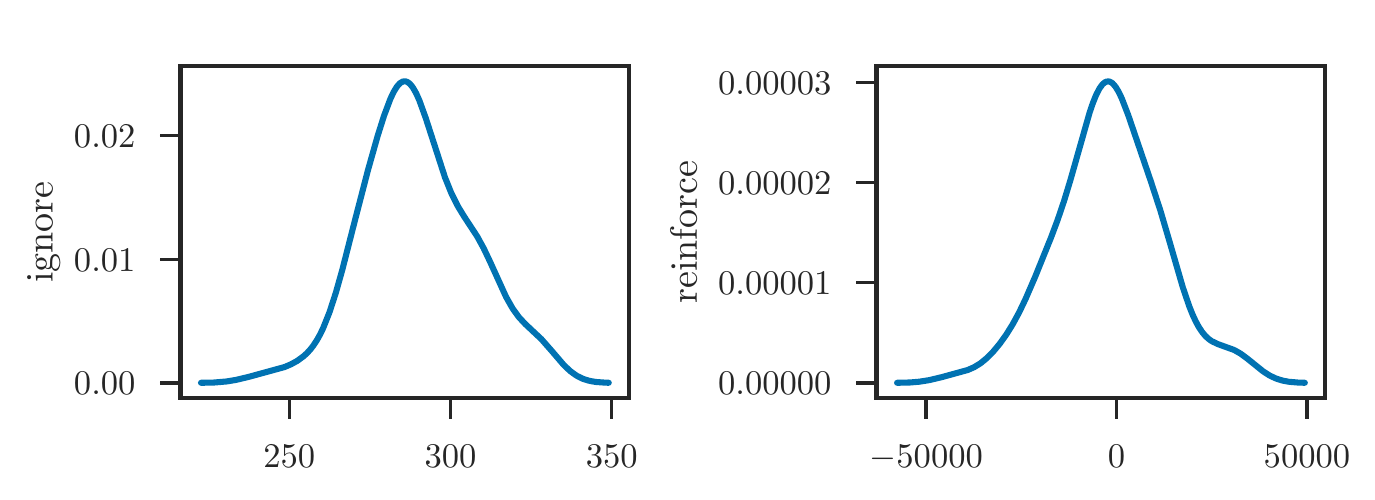}
  \caption{$T = 200$ model described in Section~\ref{sec:experiments/lgssm}. \Gls{KDE} of $\nabla_{\theta_1} \ELBO_{\text{SMC}}$ evaluated at $\theta_1 = 0.1$ with $K = 16$ using $100$ samples.}
  \label{fig:gradients/reinforce_vs_ignore}
\end{figure}

\section{Proofs for Bias \& Implications on the Proposals}
\label{sec:bias-proofs}

\begin{proof}[Derivation of~\eqref{eqn:aesmc/bias/elbo2}]
    \begin{align}
        \ELBO &= \int Q(x) \log \frac{Z_P P(x)}{Q(x)} \,\mathrm dx \\
        &= \int Q(x) \log Z_P \,\mathrm dx - \int Q(x) \log\frac{Q(x)}{P(x)} \,\mathrm dx \\
        &= \log Z_P - \KL{Q}{P}. \label{eqn:bias-proofs/basic_elbo}
    \end{align}
\end{proof}

\begin{proof}[Proof of Claim~\ref{claim:aesmc/bias/estimator}]
    Since $\hat Z_P(x) \geq 0$, $Q(x) \geq 0$ and $\int Q(x) \hat Z_P(x) \,\mathrm dx = Z_P$, we can let the unnormalized target density in Definition~\ref{claim:aesmc/bias/elbo} be $\tilde P(x) = Q(x) \hat Z_P(x)$.
    Hence, the normalized target density is $P(x) = Q(x) \hat Z_P(x) / Z_P$.
    Substituting these quantities into \eqref{eqn:aesmc/bias/elbo} and \eqref{eqn:aesmc/bias/elbo2} yields the two equalities in \eqref{eqn:aesmc/bias/elbo3}.
\end{proof}

\begin{proof}[Proof of Proposition~\ref{proposition:aesmc/bias/elbo_is}]
    ($\implies$) Substituting for $Q_{\text{IS}}(x^{1:K}) = P_{\text{IS}}(x^{1:K})$, we obtain
    \begin{align}
        \prod_{k = 1}^K q(x^k \given y) &= \frac{1}{K} \sum_{k = 1}^K \frac{\prod_{\ell = 1}^K q(x^\ell \given y)}{q(x^k \given y)} p(x^k \given y) \\ &= \frac{1}{K} \sum_{k = 1}^K \left[q(x^1 \given y) \cdots q(x^{k - 1} \given y) p(x^k \given y) q(x^{k + 1} \given y) \cdots q(x^K \given y)\right].
    \end{align}
    Integrating both sides with respect to $(x^2, \dotsc, x^K)$ over the whole support (i.e. marginalizing out everything except $x^1$), we obtain:
    \begin{align}
        q(x^1 \given y) = \frac{1}{K} \left[p(x^1 \given y) + \sum_{k = 2}^K q(x^1 \given y)\right].
    \end{align}
    Rearranging gives us $q(x^1 \given y) = p(x^1 \given y)$ for all $x^1$.

    ($\impliedby$) Substituting $p(x^k \given y) = q(x^k \given y)$, we obtain
    \begin{align}
        P_{\text{IS}}(x^{1:K}) &= \frac{1}{K} \sum_{k = 1}^K \frac{Q_{\text{IS}}(x^{1:K})}{q(x^k \given y)} p(x^k \given y) \\
        &= \frac{1}{K} \sum_{k = 1}^K Q_{\text{IS}}(x^{1:K}) \\
        &= Q_{\text{IS}}(x^{1:K}).
    \end{align}
\end{proof}

\begin{proof}[Proof of Proposition~\ref{proposition:aesmc/bias/elbo_smc}]
    We consider the general sequence of target distributions $\pi_t(x_{1:t})$ ($p_{\theta}(x_{1:t} \given y_{1:t})$ in the case of \glspl{SSM}), their unnormalized versions $\gamma_t(x_{1:t})$ ($p_{\theta}(x_{1:t}, y_{1:t})$ in the case of \glspl{SSM}), their normalizing constants $Z_t = \int \gamma_t(x_{1:t}) \,\mathrm dx_{1:t}$ ($p_{\theta}(y_{1:t})$ in the case of \glspl{SSM}), where $Z = Z_T = p(y_{1:T})$.

    ($\implies$)
    It suffices to show that $\hat Z_{\text{SMC}}(x_{1:T}^{1:K}, a_{1:T - 1}^{1:K}) = Z$ for all $(x_{1:T}^{1:K}, a_{1:T - 1}^{1:K})$ implies 1 and 2 in Proposal~\ref{proposition:aesmc/bias/elbo_smc} due to equation~\eqref{eqn:aesmc/bias/P}.

    We first prove that $\hat Z_{\text{SMC}}(x_{1:T}^{1:K}, a_{1:T - 1}^{1:K}) = Z$ for all $(x_{1:T}^{1:K}, a_{1:T - 1}^{1:K})$ implies that the weights
    \begin{align}
        w_1(x_1) &:= \frac{\gamma_1(x_1)}{q_1(x_1)} \label{eqn:weight_1} \\
        w_t(x_{1:t}) &:= \frac{\gamma_t(x_{1:t})}{\gamma_{t - 1}(x_{1:t - 1}) q_t(x_t \given x_{1:t - 1})} && \text{ for } t = 2, \dotsc, T \label{eqn:weight_t}
    \end{align}
    are constant with respect to $x_{1:t}$.

    Pick $t \in \{1, \dotsc, T\}$ and distinct $k, \ell \in \{1, \dotsc, K\}$.
    Also, pick $x_{1:t}$ and ${x'}_{1:t}$.
    Now, consider two sets of particle sets $(\bar x_{1:T}^{1:K}, \bar a_{1:T - 1}^{1:K})$ and $(\tilde x_{1:T}^{1:K}, \tilde a_{1:T - 1}^{1:K})$, illustrated in Figure~\ref{fig:bias-proofs/particle_sets}, such that
    \begin{align}
        \bar x_\tau^\kappa &=
        \begin{cases}
            {x'}_\tau & \text{ if } \kappa = \ell \text{ and } \tau < t \\
            {x'}_\tau & \text{ if } (\kappa, \tau) = (k, t) \\
            x_\tau & \text{ if } \kappa = k \text{ and } \tau < t \\
            x_\tau^\kappa & \text{ otherwise} 
        \end{cases} && \text{ for } \tau = 1, \dotsc, T, ~\kappa = 1, \dotsc, K, \\
        \bar a_\tau^\kappa &=
        \begin{cases}
            \ell & \text{ if } (\kappa, \tau) = (k, t - 1) \text{ or } (k, t) \\
            \kappa & \text{ otherwise}
        \end{cases} && \text{ for } \tau = 1, \dotsc, T - 1, ~\kappa = 1, \dotsc, K, \\
        \tilde x_\tau^\kappa &=
        \begin{cases}
            {x'}_\tau & \text{ if } \kappa = \ell \text{ and } \tau < t \\
            x_\tau & \text{ if } (\kappa, \tau) = (k, t) \\
            x_\tau & \text{ if } \kappa = k \text{ and } \tau < t \\
            x_\tau^\kappa & \text{ otherwise} 
        \end{cases} && \text{ for } \tau = 1, \dotsc, T, ~\kappa = 1, \dotsc, K, \\
        \tilde a_\tau^\kappa &=
        \begin{cases}
            \ell & \text{ if } (\kappa, \tau) = (k, t) \\
            \kappa & \text{ otherwise}
        \end{cases} && \text{ for } \tau = 1, \dotsc, T - 1, ~\kappa = 1, \dotsc, K.
    \end{align}
    \begin{figure}[!htb]
    \centering
    \begin{subfigure}{.4\textwidth}
      \centering
      \includegraphics[width=\textwidth]{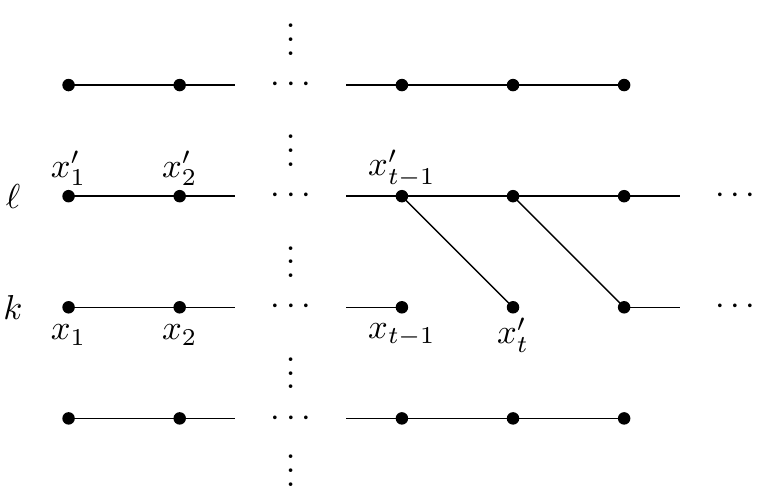}
    \end{subfigure}\hfill
    \begin{subfigure}{.4\textwidth}
      \centering
      \includegraphics[width=\textwidth]{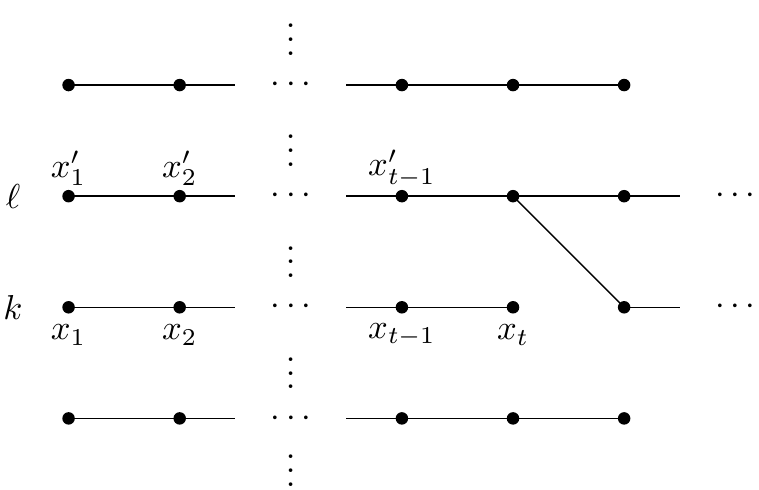}
    \end{subfigure}
    \caption{(Left) particle set $(\bar x_{1:T}^{1:K}, \bar a_{1:T - 1}^{1:K})$ and (right) particle set $(\tilde x_{1:T}^{1:K}, \tilde a_{1:T - 1}^{1:K})$. Lines indicate ancestor indices.}
    \label{fig:bias-proofs/particle_sets}
    \end{figure}

    The weights $\bar w_\tau^\kappa$ and $\tilde w_\tau^\kappa$ for the respective particle sets are identical except when $(\tau, \kappa) = (t, k)$ where
    \begin{align}
        \bar w_t^k = w_t({x'}_{1:t}), \\
        \tilde w_t^k = w_t(x_{1:t}).
    \end{align}
    Since $\hat Z(\bar x_{1:T}^{1:K}, \bar a_{1:T - 1}^{1:K}) = \hat Z(\tilde x_{1:T}^{1:K}, \tilde a_{1:T - 1}^{1:K})$, we have $w_t({x'}_{1:t}) = w_t(x_{1:t})$.  As this holds for any arbitrary $t$ and $x_{1:t}$,
    it follows that $w_t(x_{1:t})$ must be constant with respect to $x_{1:t}$ for all $t = 1, \dotsc, T$.

    Now, for $x_{1:t}$, consider the implied proposal by rearranging \eqref{eqn:weight_1} and \eqref{eqn:weight_t}
    \begin{align}
        q_1(x_1) &= \frac{\gamma_1(x_1)}{w_1} \\
        q_t(x_t \given x_{1:t - 1}) &= \frac{\gamma_t(x_{1:t})}{\gamma_{t - 1}(x_{1:t - 1}) w_t} && \text{ for } t = 2, \dotsc, T,
    \end{align}
    where $w_t := w_t(x_{1:t})$ is constant from our previous results. For this to be a normalized density with respect to $x_t$, we must have
    \begin{align}
        w_1 = \int \gamma_1(x_1) \,\mathrm dx_1 &= Z_1,
    \end{align}
    and for $t = 2, \dotsc, T$:
    \begin{align}
        w_t &= \int \frac{\gamma_t(x_{1:t})}{\gamma_{t - 1}(x_{1:t - 1})} \,\mathrm dx_t \\
        &= \frac{\int \gamma_t(x_{1:t}) \,\mathrm dx_t}{\gamma_{t - 1}(x_{1:t - 1})} \\
        &= \frac{Z_t}{Z_{t - 1}} \cdot \frac{\int \pi_t(x_{1:t}) \,\mathrm dx_t}{\pi_{t - 1}(x_{1:t - 1})}.
    \end{align}
    Since $\int \pi_{t + 1}(x_{1:t + 1}) \,\mathrm dx_{t + 1}$ and $\pi_t(x_{1:t})$ are both normalized densities, we must have $\pi_t(x_{1:t}) = \int \pi_{t + 1}(x_{1:t + 1}) \,\mathrm dx_{t + 1}$ for all $t = 1, \dotsc, T - 1$ for all $x_{1:t}$. For a given $t \in \{1, \dotsc, T - 1\}$ and $x_{1:t}$, applying this repeatedly yields
    \begin{align}
        \pi_t(x_{1:t}) = \int \pi_{t + 1}(x_{1:t + 1}) \,\mathrm dx_{t + 1}
        = \int \int \pi_{t + 2}(x_{1:t + 2}) \,\mathrm dx_{t + 2} \,\mathrm dx_{t + 1}
        = \cdots 
        = \int \pi_T(x_{1:T}) \,\mathrm dx_{t + 1:T}
    \end{align}
    such that each $\pi_t(x_{1:t})$ must be the corresponding marginal of the final target.
    We also have
    \begin{align}
        w_1(x_1) &= Z_1, \\
        w_t(x_{1:t}) &= \frac{Z_t}{Z_{t - 1}}, && t = 2, \dotsc, T, \\
        q_1(x_1) &= \pi_1(x_1) = \pi_T(x_1), \\
        q_t(x_t \given x_{1:t - 1}) &= \frac{\pi_t(x_{1:t})}{\pi_{t - 1}(x_{1:t - 1})} = \frac{\pi_T(x_{1:t})}{\pi_T(x_{1:t - 1})}, && t = 2, \dotsc, T.
    \end{align}

    ($\impliedby$)
    To complete the proof, we now simply substitute identities in 1 and 2 of Proposal~\ref{proposition:aesmc/bias/elbo_smc} back to the expression of $\hat Z(x_{1:T}^{1:K}, a_{1:T - 1}^{1:K})$ to obtain $\hat Z(x_{1:T}^{1:K}, a_{1:T - 1}^{1:K}) = Z$.
\end{proof}

\section{Experiments}
\subsection{VRNN}
\label{sec:appendix_vrnn}
In the following we give the details of our \gls{VRNN} architecture.
The generative model is given by:

\begin{equation}
  p(x_{1:T}, h_{0:T}, y_{1:T}) = p(h_0) \prod_t p(x_t|h_{t-1})p(y_t|h_{t-1}, x_t)p(h_t|h_{t-1},x_t,y_t)
\end{equation}

where
\begin{equation}
\begin{split}
p(h_0) & = \mathrm{Normal}(h_0; 0, I)\\
p(x_t|h_{t-1}) & = \mathrm{Normal}(x_t;\mu^x_\theta(h_{t-1}), {\sigma^x_\theta}(h_{t-1})^2)\\
p(y_t|h_{t-1}, x_t) & = \mathrm{Bernoulli}(y_t;\mu^y_\theta(\varphi^x_\theta(x_t), h_{t-1}))\\
p(h_t|h_{t-1},x_t,y_t) & = \delta_{f(h_{t-1},\varphi^x_\theta(x_t),\varphi^y_\theta(y_t))}(h_t)
\end{split}
\end{equation}

and the proposal distribution is given by

\begin{equation}
p(x_t|y_t, h_{t-1}) = \mathrm{Normal}(x_t;\mu^p_\phi(\varphi^y_\phi(y_t), h_{t-1}), {\sigma^p_\phi}^2(\varphi^y_\phi(y_t), h_{t-1}))
\end{equation}

The functions $\mu^x_\theta$ and $\sigma^x_\theta$ are computed by networks with two fully connected layers of size 128 whose first layer is shared.
$\varphi^x_\theta$ is one fully connected layer of size 128.

For visual input, the encoding $\varphi^y_\theta$ is a convolutional network with conv-4x4-2-1-32, conv-4x4-2-1-64, conv-4x4-2-1-128 where conv-wxh-s-p-n denotes a convolutional network with $n$ filters of size $w\times h$, stride $s$, padding $p$. Between convolutions we use leaky ReLUs with slope 0.2 as nonlinearity and batch norms. The decoding $\mu^y_\theta$ uses transposed convolutions of the same dimensions but in reversed order, however with stride $s=1$ and padding $p=0$ for the first layer.

A Gated Recurrent Unit (GRU) is used as RNN and if not stated otherwise ReLUs are used in between fully connected layers.

For the proposal distribution, the functions $\mu^p_\phi$ and $\sigma^p_\phi$ are neural networks with three fully connected layers of size 128 that are sharing the first two layers.
Sigmoid and softplus functions are used where values in $(0,1)$ or $\mathbb{R}^{+}$ are required.
We use a minibatch size of 25.

For the moving agents dataset we use \gls{ADAM} with a learning rate of $10^{-3}$.

A specific feature of the \gls{VRNN} architecture is that the proposal and the generative model share the component $\varphi^y_{\phi,\theta}$.
Consequently, we set $\phi=\theta$ for the parameters belonging to this module and train it using gradients for both $\theta$ and $\phi$.
\subsection{Moving Agents}
\label{sec:appendix_moving_agents}

In Figure~\ref{fig:experiments/max/visualisation} we investigate the quality of the generative model by comparing visual predictions. We do so for models learned by \gls{IWAE} \emph{(top)} and \gls{AESMC} \emph{(bottom)}. The models were learned using ten particles but for easier visualization we only predict using five particles.

The first row in each graphic shows the ground truth. The second row shows the averaged predictions of all five particles. The next five rows show the predictions made by each particle individually.

The observations (i.e. the top row) up to $t=19$ are shown to the model. Up to this timestep the latent values $x_{0:19}$ are drawn from the proposal distribution $q(x_t|y_t, h_{t-1})$.
From $t=20$ onwards the latent values $x_{20:37}$ are drawn from the generative model $p(x_t|x_{t-1})$.
Consequently, the model predicts the partially occluded, stochastic movement over 17 timesteps into the future.

We note that most particles predict a viable future trajectory.
However, the model learned by \gls{IWAE} is not as consistent in the quality of its predictions, often 'forgetting' the particle.
This does not happen in every predicted sequence but the behavior shown here is very typical.
Models learned by \gls{AESMC} are much more consistent in the quality of their predictions.
\begin{figure}[!htb]
\centering
\begin{subfigure}{\textwidth}
  \centering
  \includegraphics[width=.95\linewidth]{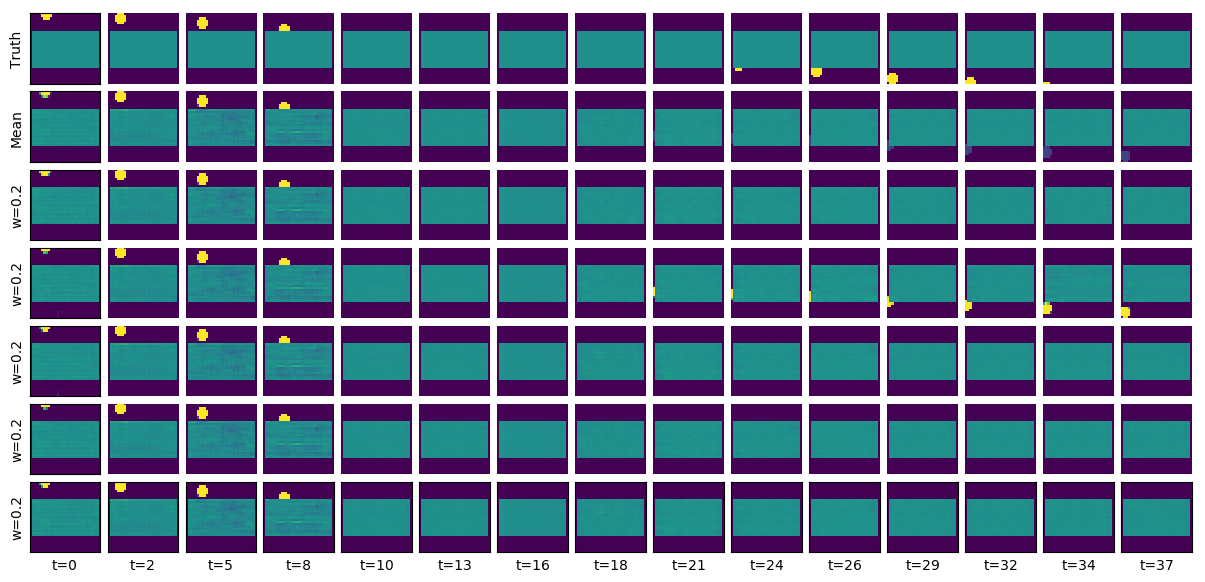}
\end{subfigure}
\begin{subfigure}{\textwidth}
  \centering
  \includegraphics[width=.95\linewidth]{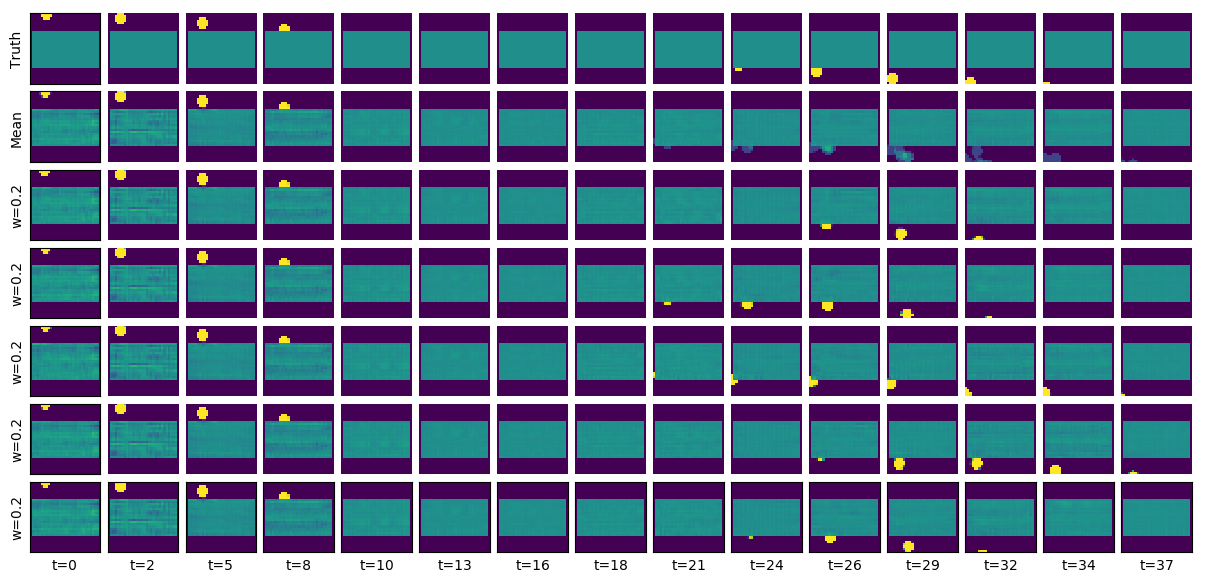}
\end{subfigure}
\caption{Visualisation of the learned model. Ground truth observations (top row in each sub figure) are only revealed to the algorithm up until t=19 inclusive. The second row shows the prediction averaged over all particles, all following rows show the prediction made by a single particle. \emph{(Top)} \gls{IWAE}. \emph{(Bottom)} \gls{AESMC}.}
\label{fig:experiments/max/visualisation}
\end{figure}

\subsection{Optimizing Only Proposal Parameters}
\label{sec:additional-experiments}
\begin{figure}[htb!]
  \centering
  \begin{subfigure}[!t]{.49\textwidth}
    \centering
    \includegraphics[width=\textwidth]{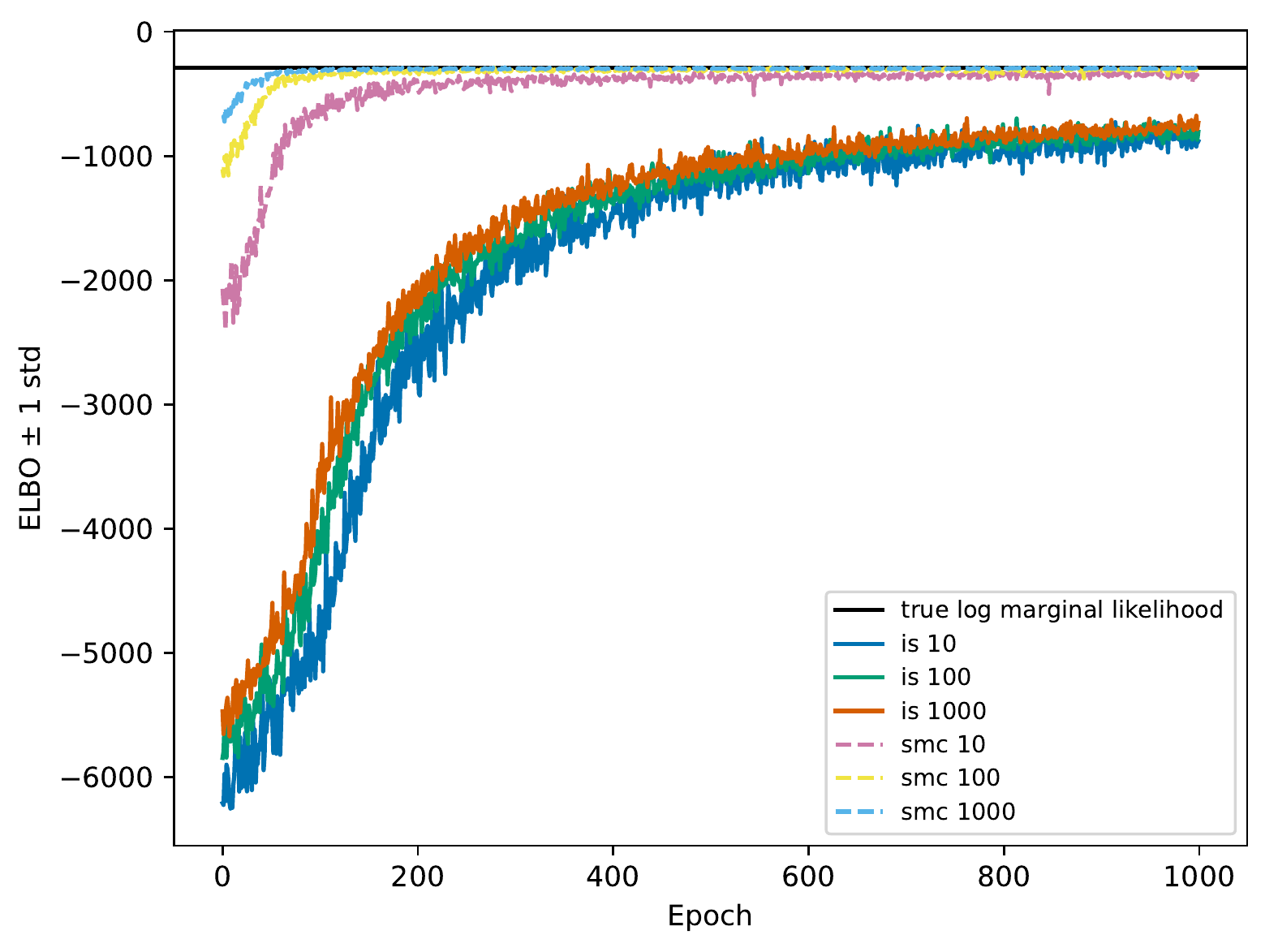}
  \end{subfigure}
  \begin{subfigure}[!t]{.49\textwidth}
    \centering
    \includegraphics[width=\textwidth]{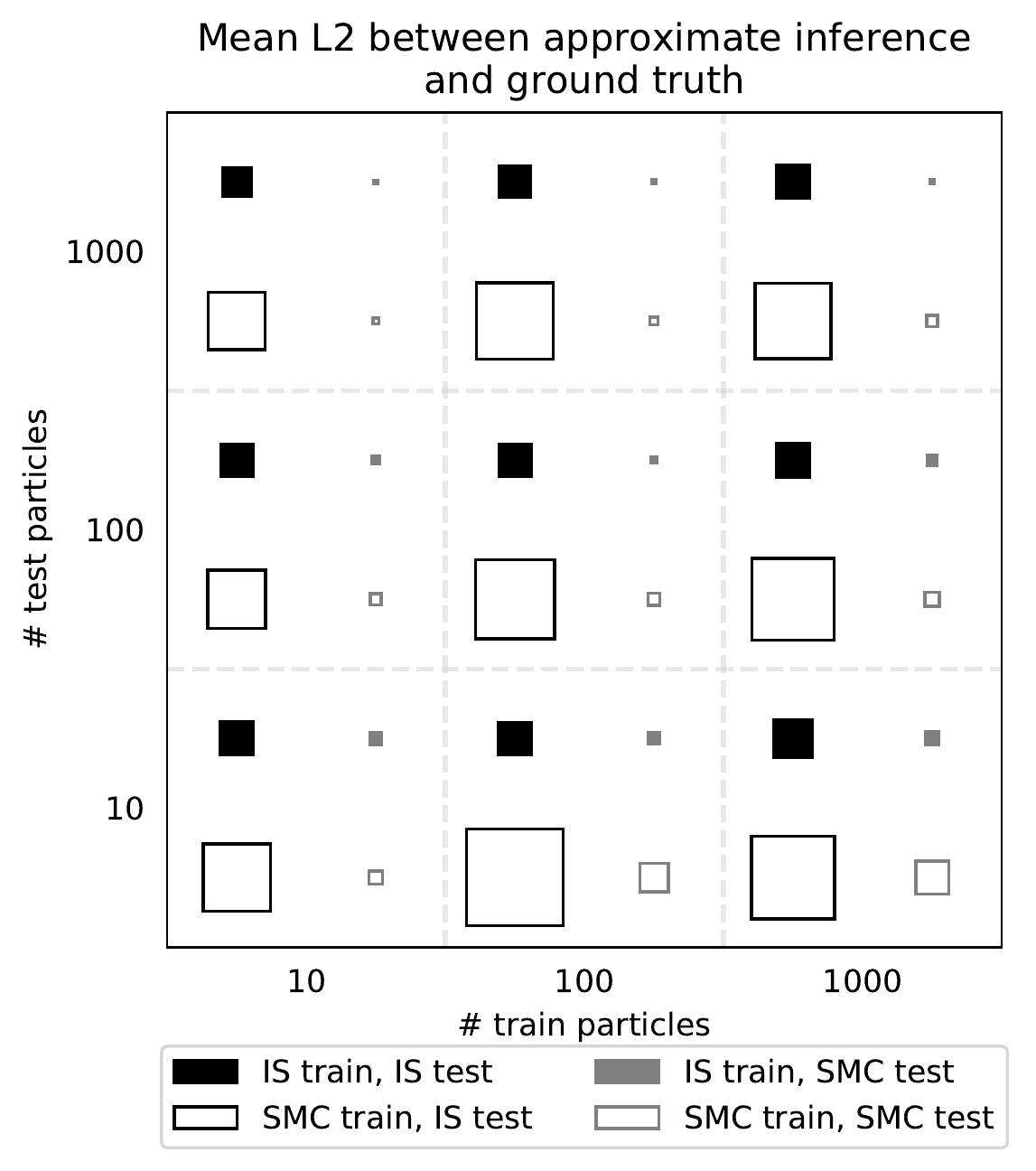}
  \end{subfigure}
  \caption{\emph{(Left)} Optimizing $\ELBO$ with respect to $\phi$ for \gls{LGSSM}. \emph{(Right)} The lengths of the squares are proportional (with a constant factor) to $\sqrt{\sum_{t = 1}^T (\mu_t^{\text{kalman}} - \mu_t^{\text{approx}})^2}$ which is a proxy for inference quality of $\phi$ described in the main text. The larger the square, the worse the inference.}
  \label{fig:additional_experiments/lgssm/inference_test_and_elbos}
\end{figure}

We have run experiments where we optimize various $\ELBO$ objectives with respect to $\phi$ with $\theta$ fixed in order to see how various objectives have an effect on proposal learning.
In particular, we train $\ELBO_{\text{IS}}$ and $\ELBO_{\text{SMC}}$ with number of particles $K \in \{10, 100, 1000\}$.
Once the training is done, we use the trained proposal network to perform inference using both \gls{IS} and \gls{SMC} with number of particles $K_{\text{test}} \in \{10, 100, 1000\}$.

In Figure~\ref{fig:additional_experiments/lgssm/inference_test_and_elbos}, we see experimental results for the \gls{LGSSM} described in Section~\ref{sec:experiments/lgssm}.
We measure the quality of the inference network using a proxy $\sqrt{\sum_{t = 1}^T (\mu_t^{\text{kalman}} - \mu_t^{\text{approx}})^2}$ where $\mu_t^{\text{kalman}}$ is the true marginal mean $\E_{p(x_{1:T} \given y_{1:T})}[x_t]$ obtained from the Kalman smoothing algorithm and $\mu_t^{\text{approx}} = \left(\sum_{k = 1}^K w_T^k x_t\right) / \left(\sum_{k = 1}^K w_T^k\right) $ is an approximate marginal mean obtained from the proposal parameterized by $\phi$.

We see that if we train using $\ELBO_{\text{SMC}}$ with $K_{\text{train}} = 1000$, the performance for inference using \gls{SMC} (with whichever $K_{\text{test}} \in \{10, 100, 1000\}$) is worse than if we train with $\ELBO_{\text{IS}}$ with any number of particles $K_{\text{train}} \in \{10, 100, 1000\}$.
Examining the other axes of variation:
\begin{itemize}
  \item Increasing $K_{\text{test}}$ (moving up in Figure~\ref{fig:additional_experiments/lgssm/inference_test_and_elbos}~(Right)) improves inference.
  \item Increasing $K_{\text{train}}$ (moving to the right in Figure~\ref{fig:additional_experiments/lgssm/inference_test_and_elbos}~(Right)) worsens inference.
  \item Among different possible combinations of (training algorithm, testing algorithm), (\gls{IS}, \gls{SMC}) $\succ$ (\gls{SMC}, \gls{SMC}) $\succ$ (\gls{IS}, \gls{IS}) $\succ$ (\gls{SMC}, \gls{IS}), where we use ``$a \succ b$'' to denote that the combination $a$ results in better inference than combination $b$.
\end{itemize}
\end{document}